\documentclass[11pt,english]{article}
\usepackage{hyperref}
\usepackage{babel}
\usepackage{blindtext}

\usepackage{enumerate}
\usepackage{amsmath,amssymb, latexsym,graphicx}
\usepackage{url}
\usepackage{algorithm}
\usepackage{color}
\usepackage{amssymb}
\usepackage{mathtools}
\usepackage{graphicx}
\usepackage[algo2e,ruled,vlined,linesnumbered]{algorithm2e}
\usepackage{subcaption}

\newcommand{\sca}{0.31}

\usepackage[papersize={8.5in,11in}]{geometry}
\usepackage{fullpage}
\graphicspath{ {img/} }
\usepackage{subcaption}
\usepackage{hyperref}

\newcommand*{\QEDW}{\hfill\ensuremath{\square}}%

\providecommand{\norm}[1]{\left\lVert#1\right\rVert}

\newcommand{\br}[1]{\left\{#1\right\}}

\newcommand{\eps}{\varepsilon}
\newcommand{\dist}{\mathrm{dist}}
\newcommand{\range}{\mathrm{range}}

\newcommand{\comment}[1]{}
\newcommand{\REAL}{\ensuremath{\mathbb{R}}}

\newcommand{\abs}[1]        {\left| #1\right|}

\DeclareMathOperator*{\argmin}{arg\,min}

\newcommand{\pr}{\mathrm{pr}}

\newcommand{\smi}{\sum_{i=1}^n}

\newcommand{\sqi}{\displaystyle{\sum_{q\in P}}}
\newcommand{\spi}{\sum_{i=1}^n}

\newcommand{\constalaa} {6z}

\newenvironment{proof}{\noindent\normalfont {\bf Proof}.\ }{\QEDW \par\vskip 4mm\par}

 \newtheorem{theorem}{Theorem}[section]
 
 \newtheorem{lemma}[theorem]{Lemma}
 
\newtheorem{definition}[theorem]{Definition}

 \newtheorem{observation}[theorem]{Observation}

\newcommand{\Tau}{\eps}

\title{\LARGE \bf Tight Sensitivity Bounds For Smaller Coresets}

\author{
  Alaa Maalouf\\
  \texttt{Alaamalouf12@gmail.com}
  \and
  Adiel Statman\\
  \texttt{statman.adiel@gmail.com}
  \and
  Dan Feldman\\
  \texttt{dannyf.post@gmail.com}
}
\date{
	\normalfont{The Robotics and Big Data Lab,\\Department of Computer Science,\\University of Haifa,\\Haifa, Israel\\}
    \today
}

\begin{document}
\maketitle

\begin{abstract}
An $\eps$-coreset for Least-Mean-Squares (LMS) of a matrix $A\in\REAL^{n\times d}$ is a small weighted subset of its rows that approximates the sum of squared distances from its rows to every affine $k$-dimensional subspace of $\REAL^d$, up to a factor of $1\pm\eps$.
Such coresets are useful for hyper-parameter tuning and solving many least-mean-squares problems such as low-rank approximation ($k$-SVD), $k$-PCA, Lassso/Ridge/Linear regression and many more. Coresets are also useful for handling streaming, dynamic and distributed big data in parallel. With high probability, non-uniform sampling based on upper bounds on what is known as importance or sensitivity of each row in $A$ yields a coreset. The size of the (sampled) coreset is then near-linear in the total sum of these sensitivity bounds.

We provide algorithms that compute provably \emph{tight} bounds for the sensitivity of each input row.

It is based on two ingredients: (i) iterative algorithm that computes the exact sensitivity of each point up to arbitrary small precision for (non-affine) $k$-subspaces, and (ii) a general reduction of independent interest from computing sensitivity for the family of affine $k$-subspaces in $\REAL^d$ to (non-affine) $(k+1)$- subspaces in $\REAL^{d+1}$.

Experimental results on real-world datasets, including the English Wikipedia documents-term matrix, show that our bounds provide significantly smaller and data-dependent coresets also in practice. Full open source is also provided.
%
\end{abstract}

\section{Introduction}
\paragraph{Motivation.}
Least mean squares solvers are fundamental tools in all the data science fields such as machine learning, computer science and statistics. They are also the building blocks of more involved techniques such as deep learning and signal processing~\cite{mandic2004generalized,widrow1977stationary}.  As explained in~\cite{maalouf2019fast}, this family include Singular Value Decomposition (SVD), Principle Component Analysis (PCA), linear regression, Lasso and Ridge regression, Elastic net, and many more~\cite{golub1971singular,jolliffe2011principal,hoerl1970ridge,seber2012linear,zou2005regularization,tibshirani1996regression,safavian1991survey}. First closed form solutions for problems such as linear regression were published by e.g. Pearson~\cite{pearson1900x} around 1900 but were probably known before. Nevertheless, today they are still used extensively as building blocks in both academy and industry for normalization~\cite{liang2013distributed,kang2011scalable,afrabandpey2016regression}, spectral clustering~\cite{peng2015robust}, graph theory~\cite{zhang2018understanding}, prediction~\cite{copas1983regression,porco2015low}, dimensionality reduction~\cite{laparra2015dimensionality}, feature selection~\cite{gallagher2017cross} and many more; see more examples in~\cite{golub2012matrix}.


Important special case is the low rank approximation of an $n\times d$ real matrix $A$ that can be computed via $k$-SVD (Singular Value Decomposition). Which is the linear (non-affine) $k$-dimensional subspace that minimizes its sum of squared distances over the rows of $A$ for a given integer $k\geq1$, i.e.,
\[
\argmin_{X\in\REAL^{d\times k},X^TX=I } \smi \norm{A_i - A_iXX^T}^2_2,
\]
where $A_i$ is the $i$th row of the matrix $A$ for every integer $i\in \br{1,\cdots , n}$.

More generally, the $k$-PCA is the \emph{affine} $k$-subspace that minimizes the sum of squared distances from the rows of $A$ to it over every $k$-subspace that may be translated from the origin of $R^d$. Formally, an affine $k$-subspace is represented by an orthogonal matrix $X\in \REAL^{d\times k}$ and a vector $\ell\in \REAL^d$ that represent the translation of the subspace from the origin. Hence, we wish to compute:
\[
\argmin_{\substack{ \ell\in \REAL^d \\ X\in\REAL^{d\times k} , X^TX=I }} \smi \norm{(A_i-\ell^T) - (A_i-\ell^T)XX^T}^2_2.
\]


Finally we have the Least-Mean-Squares solvers that gets as input an $n\times d$ real matrix $A$, and another $n$-dimensional real vector $b$ (possibly the zero vector), and aims to minimize the sum of squared distances from the rows (points) of $A$ to some hyperplane that is represented by its normal or vector of $d$ coefficients $x$, that is constrained to be in a given set $X\subseteq\REAL^d$:
\[
\argmin_{x\in X}f(\norm{Ax-b}_2)+g(x).
\]
Here, $g$ is called a \emph{regularization term}. For example: in linear regression $X=\REAL^d$, $f(x)=x^2$ and $g(x)=0$ for every $x\in X$. In Lasso $f(y) = y^2$ and $g(y) = \alpha \cdot \norm{x}_1$ for every $y\in \REAL^d$ and $\alpha>0$.
Such LMS solvers can be computed via the covariance matrix $A^TA$. For example, the solution to linear regression of minimizing $\norm{Ax-b}_2$ is $(A^TA)^{-1}A^Tb$.
\subsection{Coresets}
 For a huge amount of data, those algorithms/sovlers are much time consuming: while in theory the running time is usually $O(nd^2)$, the constants that are hidden in the $O(\cdot)$ notation are significantly large.
 Another problem with such algorithms/sovlers is that we may not be able to use them for big data  on standard machines, since there is no enough memory to provide the relevant computations. 

A modern tool to handle these type of problems, is a data summarization for the input that is sometimes called \emph{coresets}. Coresets also allow us to boost the running time of those algorithms/solvers while using less memory. 

As explained at~\cite{unpublishedDanny}, coresets are especially useful to (a) learn unbounded streaming data that cannot fit into main memory, (b) run in parallel on distributed data among thousands of machines, (c) use low communication between the machines, (d) apply real-time computations on a device, (e) handle privacy and security issues, (f) compute constrained optimization on a coreset that was constructed independently of these constraints and of course boost there running time.

In the context of the $k$-SVD problem, an $\eps$-coreset for a matrix $A\in\REAL^{n\times d}$ is a matrix $C\in \REAL^{m\times d}$ where $m \ll n$, which guarantees that the sum of the squared distances from any linear (non-affine) $k$-subspace to the rows of $C$ will be approximately equal to the  sum of the squared distances from the same $k$-subspace to the rows of $A$, up to a $(1\pm\eps)$ multiplicative factor, i.e., for any matrix $X\in\REAL^{d\times k}$  such that $X^TX=I$ we have,
$$ \abs{\smi \norm{A_i - A_iXX^T}^2_2 -  \sum_{i=1}^m \norm{C_i - C_iXX^T}^2_2} \leq \eps \smi \norm{A_i - A_iXX^T}^2_2.$$

In the $k$-PCA problem, an $\eps$-coreset for the matrix $A$ is a matrix $C\in \REAL^{m\times d}$ such that for every vector $\ell \in \REAL^d$  and a matrix $X\in\REAL^{d\times k}$ where $X^TX=I$ we have:
\[
\abs{ \smi \norm{(A_i-\ell^T) - (A_i-\ell^T)XX^T}^2_2 - \sum_{i=1}^m\norm{(C_i-\ell^T) - (C_i-\ell^T)XX^T}^2_2} \leq \eps\smi \norm{(A_i-\ell^T) - (A_i-\ell^T)XX^T}^2_2.
\]
 
The dimension of the subspace $k$ is a crucial parameter, and of course may determines the size of the coreset and the time complexity of the algorithm.
Such coresets are useful, for example, for many NLP applications in which a Word Embedding model needed to be produced out of a large database, see \cite{mikolov2013distributed,mikolov2013efficient,pennington2014glove}.

Considering  the least means squared problems. Given a matrix  $A\in \REAL^{n\times d}$ and a $n$-dimensional real vector $b$, a coreset for the pair $(A,b)$ in the LMS problem that is represented by the functions $f$ and $g$ as explained before, is a matrix $C\in \REAL^{m\times d}$ where $m \ll n$ and a vector $y\in \REAL^{m}$, such that  for every hyperplane that is represented by its normal or vector of $d$ coefficients $x$ we have,
$$\abs{f(\norm{Ax-b}_2)- f(\norm{Cx-y}_2)} \leq \eps \big( f(\norm{Ax-b}_2)+g(x) \big).$$
Usually $g$ is  a non-negative function (e.g., $g(x) = \norm{x}_2^2$ in Ridge, $g(x) = \norm{x}_1$ in Lasso, and $g(x)=0$ in linear regression). Hence for those cases, $(C,y)$ is a coreset for $(A,b)$  if it satisfies: 
\[
\abs{f(\norm{Ax-b}_2)- f(\norm{Cx-y}_2) }\leq \eps f(\norm{Ax-b}_2),
\]
for every $x\in \REAL^d$.

\subsection{Coreset constructions}
One type of coresets, sometimes called \emph{sketch}, consists on linear combinations of the input points. These coresets use techniques such as Random projections~\cite{DBLP:journals/corr/CohenEMMP14}, JL-Lemma~\cite{sarlos2006improved},SVD~\cite{DBLP:journals/corr/abs-1807-04518} etc. 
However, in this paper we consider only coresets that are \emph{subset} of their input points (rows of the input matrix), up to a multiplicative weight (scaling).  

As explained in~\cite{unpublishedDanny} and~\cite{DBLP:journals/corr/BargerF15}, the advantages of such coresets are: (i) preserved sparsity of the input, (ii) interpretability, (iii) coreset may be used (heuristically) for other problems, (iv)~less numerical issues that occur when non-exact linear combination of points are used.

\paragraph{Sensitivity sampling.}
Over the recent decades many algorithms were suggested to compute such coresets. 
One of the common technique, both in theory and practice, that yields fast and provably good coresets is the approach of non-uniform sampling, sensitivity sampling~\cite{langberg2010universal,DBLP:journals/corr/BravermanFL16}. The sensitivity of a row $p$ in the input matrix $A$ is a number $s(p)\in [0,1]$ that represents how much this row is `important' in this dataset with respect to the desired optimization problem. The motivation for defining sensitivity is the following. Suppose that we have an upper bound $s'(p)\geq s(p)$ for every row $p$ in the matrix $A$, and we use it to sample $m\geq 1$ rows from $A$ where every row $p$ is picked with probability $\pr(p)$ that is proportional to $s'(p)$ and is assigned a weight of $w(p)=\frac{1}{m\cdot pr(p)}$. Then sampling such $\log(d)T/\eps^2$ i.i.d. rows would yield an $\eps$-coreset where $T=\smi s'(A_i)$ is called the total sensitivity bound ($A_i$ is the $i$th row of $A$ for every integer $i\in \br{1,\cdots,n}$ as previously defined).

\paragraph{Sensitivity bounds. }One of the main challenges in constructing coresets is to bound the corresponding sensitivity of each point, which is $s(\cdot)$ in~\eqref{sensreduction}.
While $1$ is a trivial bound for $s(\cdot)$, it would give a coreset of size $|C|=O(n)$ as $T=n$ in this case. In the $k$-SVD problem, for the case of $k=d-1$, as shown at~\cite{yang2017weighted} the sensitivity is also known as leverage score and can be easily bounded by $s'(p)=\norm{u}_2^2$ where $u$ is the corresponding row for $p$ in the matrix $U$ such that $A=UDV^T$ is the thin SVD of the matrix $A$; see Definition~\ref{thinsvddef}, and the sum of sensitivities is exactly $T$. It is easy to prove that this bound is tight in the sense that $s'(p)=s(p)$. For the case $k\leq d-2$, sensitivity bounds are also known whose total sensitivity is $T=O(d)$, by projecting the points on an optimal (or approximated) $k$-subspace an computing the sensitivity of the projected point as shown at~\cite{DBLP:journals/corr/abs-1209-4893}. However, unlike the previous case, these bounds are not tight, as proved in the experimental results of this paper. In the $k$-PCA problem, a tight sensitivity bound for the case $k=0$ (i.e., the $1$-mean problem) was suggested at~\cite{tremblay2018determinantal}, however there is no tight bound for the other cases.

\subsection{Our contribution}
In this work we suggest:
\begin{enumerate} [(i)]
\setlength{\itemsep}{-1pt}
\item  the first algorithm that computes tight sensitivity bounds for the family of (non-affine) $k$-subspaces; see Algorithm~\ref{alg:tight1}. The algorithm is iterative and returns the exact sensitivity $s(p)$ up to arbitrarily small constant. The convergence rate is linear.
\item generalization of the above algorithm for the family of affine $k$-dimensional subspaces of $\REAL^d$. This is by reduction to a problem of computing sensitivity bounds of a new set of points in $\REAL^{d+1}$ for the family of (non-affine) $(k+1)$- subspaces in $\REAL^{d+1}$. See Theorem~\ref{lemma:pca}.
\item experimental results on real-world datasets, including the English Wikipedia documents-term matrix, that show that our bounds provide significantly smaller and data-dependent coresets also in practice.
\item full open source code.
\end{enumerate} 
While our sensitivity bounds are tight for the family of affine (or non-affine) $k$-subspaces in $\REAL^d$ they are no longer tight if we consider only subset of this family of subspaces. Nevertheless, they provide better upper bounds for these problems or query spaces, compared to existing upper bounds that also ignore these constraints and regularization terms. More precisely, the worst case sensitivity is $O(k)$ in both cases, but our bounds are tighter.

\section{Preliminaries}
In the this section we give our notations and definitions that will be used through the paper. We also explain the relation between the notion of total sensitivity and coreset size while relying on Theorem~5.5 in~\cite{DBLP:journals/corr/BravermanFL16}.

\begin{paragraph}{Notations.}For integers $d,n\geq 1$, we denote by $0_d$ the origin of $\REAL^d$. The set $\REAL^{n\times d}$ denote the union over every $n\times d$ real matrix. For a matrix $A\in\REAL^{n\times d}$ the Frobenius norm $\norm{A}_F$ is the squared root of its sum of squared entries, and $Tr(A)$ denotes its trace. A \emph{weighted set} of $n$ points in $\REAL^d$ is a pair $(P,w)$ where $P=\br{p_1,\cdots,p_n}$ is an ordered set in $\REAL^d$, and $w: P\to [0,\infty)$ is called a \emph{weight function}.



For an integer $k \in \br{0,\cdots,d-1}$, a \emph{$k$-subspace} is a shorthand for a $k$-dimensional linear (non-affine) subspace of $\REAL^d$ (i.e., it contains the origin). An \emph{affine $k$-subspace} ($k$-flat) is a translation of a $k$-subspace, i.e.,  that may not contain the origin.
For every point $p\in \REAL^d$ and an affine $k$-subspace $S$ of $\REAL^d$, we define $\mathrm{proj}(p,S)=\argmin_{x\in S} \norm{p-x}_2$ to be the projection of the point $p$ onto the affine $k$-subspace $S$ and $\dist(p,S)=\min_{x\in S}\norm{p-x}_2=\norm{p-\mathrm{proj}(p,S)}_2$ to be the Euclidean distance between the point $p$ to its closest point on $S$. This distance to the power of $z\geq1$ is denoted by $D_z(p,S) =\dist^z(p,S)$ and for brevity, we define $D(p,S) =D_2(p,S) =\dist^2(p,S)$.
\end{paragraph}

\begin{definition} [\emph{Additive $\Tau$-approximation}]
Let $d$ be an integer, $\eps\in (0,1)$ be an error parameter, $f:\REAL^d \to \REAL$ be a function and $s\in \REAL$ be a real number. We call $s$ an \emph{additive $\Tau$-approximation} for $f$ if and only if
$$\sup_{x\in \REAL^d}f(x) \leq s \leq \sup_{x\in \REAL^d}f(x)+\eps .$$
\end{definition}

\begin{definition} [\emph{Thin SVD}] \label{thinsvddef}
Let $n,d$ be two integers. Let $\mathbf{A}\in \REAL^{n\times d}$ be a matrix and let the integer $r$ be its rank. We call $\mathbf{A}=UDV^T$ the thin Singular Value Decomposition of $\mathbf{A}$. That is, $U\in \REAL^{n\times r}$, $V\in \REAL^{d\times r}$, $U^TU=I$, $V^TV=I$ and $D$ is a diagonal matrix.
\end{definition}

\begin{definition} [\textbf{Definition 4.2 in~\cite{DBLP:journals/corr/BravermanFL16}}] \label{def:querySpace}
Let $(P,w)$ be a weighted set of $n$ points in $\REAL^d$. Let $Q$ be a function that maps every set $C\subseteq P$ to a corresponding set $Q(C)$, such that $Q(T) \subseteq Q(C)$ for every $T\subseteq C$. Let $f:P\times Q(P) \to \REAL$ be a cost function. The tuple $(P,w,Q,f)$ is called a \emph{query space}.
\end{definition}

\begin{definition} [\textbf{Definition 4.5 in~\cite{DBLP:journals/corr/BravermanFL16}}] \label{def:VC}
For a query space $(P,w,Q,f)$, $q\in Q(P)$ and $r\in [0,\infty)$ we define
\[
\range(q,r) = \br{p\in P \mid w(p)\cdot f(p,q) \leq r}.
\]
The dimension of $(P,w,Q,f)$ is the smallest integer $d'$ such that for every $C\subseteq P$ we have
\[
\left| \br{\range(q,r) \mid q\in Q(C), r\in [0,\infty)} \right| \leq |C|^{d'}.
\]
\end{definition}

\begin{theorem} [\textbf{Theorem 5.5 in~\cite{DBLP:journals/corr/BravermanFL16}}]\label{braverman}
Let $(P,w,Q,f)$ be a query space; see Definition~\ref{def:querySpace}, where $f$ is a non-negative function. Let $s:P\to [0,\infty)$ such that
\[
\sup_{q\in Q} \frac{w(p)f(p,q)}{\sum_{p\in P} w(p)f(p,q)} \leq s(p),
\]
for every $p\in P$ and $q\in Q(P)$ such that the denominator is non-zero. Let $t = \sum_{p\in P} s(p)$ and let $d'$ be the dimension of the query space $(P,w,Q,f)$; See Definition~\ref{def:VC}. Let $c \geq 1$ be a sufficiently large constant and let $\varepsilon, \delta \in (0,1)$. Let $C$ be a random sample of
\[
|C| \geq \frac{ct}{\varepsilon^2}\left(d'\log{t}+\log{\frac{1}{\delta}}\right)
\]
points from $P$, such that $p$ is sampled with probability $s(p)/t$ for every $p\in P$. Let $u(p) = \frac{t\cdot w(p)}{s(p)|C|}$ for every $p\in C$. Then, with probability at least $1-\delta$, for every $q\in Q$ it holds that
\[
(1-\varepsilon)\sum_{p\in P} w(p)\cdot f(p,q) \leq \sum_{p\in C} u(p)\cdot f(p,q) \leq (1+\varepsilon)\sum_{p\in P} w(p)\cdot f(p,q).
\]
\end{theorem}
\begin{paragraph} {Smaller total sensitivity implies smaller coreset size.}
Given $(P,w)$ a weighted set of $n$ points in $\REAL^d$ and given also the sensitivity $s:P\to [0,\infty)$ of each point as defined in~\ref{braverman}, in order to obtain a coreset that guarantees $(1\pm\varepsilon)$ multiplicative error with probability at least $1-\delta$ we have to sample $O(\frac{ct}{\varepsilon^2}\left(d'\log{t}+\log{\frac{1}{\delta}}\right))$ points from $P$ where $t=\spi s(p)$ is the sum of sensitivity over all the points in $P$. Thus the smaller the sensitivity bound $s(p)$ of each point $p\in P$ the smaller the total sensitivity $t$ and the smaller is the size of the coreset needed.

\end{paragraph}
\section{Sensitivity of  Non-affine $k$-subspaces}
Let $S'$ be a (non-affine) $k$-subspace of $\REAL^d$. Every such subspace $S'$ corresponds to a column space of a matrix $X\in \REAL^{d\times k}$ whose columns are orthonormal ($X^TX=I$). Let $(P,w)$ be a weighted set of $n\geq 1$ points in $\REAL^d$.
Let $\mathbf{P}\in \REAL^{n\times d}$ denote the matrix whose $i$th row is the $i$th point of $P$ multiplied by the square root of its weight,
$
\mathbf{P}=\begin{bmatrix}
\sqrt{w(p_1)}p_1 &
\cdots
\sqrt{w(p_n)}p_n
\end{bmatrix}^T
$
and let $\mathbf{p} = \sqrt{w(p)}p^T$ for every $p\in P$.
 The projection of the rows of $\mathbf{P}$ onto $S'$ is $\mathbf{P}X\in\REAL^{n\times k}$ using the column base of $X$, and $\mathbf{P}XX^T\in\REAL^{n\times d}$ in $\REAL^d$. Hence, for every $p\in P$, the weighted squared distance from $p$ to $S'$ is
\begin{align}
w(p)\cdot D(p,S') =  w(p) \norm{p^T-p^TXX^T}_2^2 = \norm{\sqrt{w(p)} \cdot p^T- \sqrt{w(p)} \cdot p^TXX^T}_2^2 =\norm{\mathbf{p} - \mathbf{p}XX^T}_2^2. \nonumber
\end{align}

By letting $Y\in \REAL^{d\times(d-k)}$ be the matrix that spans the orthogonal complement subspace of $S'$ (i.e., $Y^TY=I$ and $[X,Y][X,Y]^T=I$), we obtain
\begin{align*}
\mathbf{p} = \mathbf{p}I =\mathbf{p}(XX^T + YY^T) = \mathbf{p}XX^T + \mathbf{p}YY^T,
\end{align*}
and by subtracting $\mathbf{p}XX^T$ from both sides and applying squared norm we get that
\begin{align*}
w(p)D(p,S')=\norm{\mathbf{p}- \mathbf{p}XX^T}_2^2 = \norm{\mathbf{p}YY^T}^2_2 = \norm{\mathbf{p}Y}^2_2.
\end{align*}
From the last equality we get that the sum of squared distance from the set $P$ to the $k$-subspace $S'$ is
\begin{align*}
\sqi  w(q)\cdot D(q,S') = \sqi \norm{\mathbf{q}Y}^2_2 =  \norm{\mathbf{P}Y}^2_F.
\end{align*}

Thus by letting $\mathcal{S}_{d}$ be the set of all (non-affine) $k$-subspaces of $\REAL^d$, we get that the sensitivity of a point $p\in P$ in the query space $(P,w,\mathcal{S}_{d},D)$ is:
\begin{align}
s(p) = \sup_{S'\in \mathcal{S}_{d}} \frac{w(p)D(p,S')}{\sqi w(q)D(q,S')} =
 \sup_{Y\in \REAL^{d\times (d-k)},  Y^TY=I} \frac {\norm{\mathbf{p}Y}_2^2}{\norm{\mathbf{P}Y}_F^2},\label{sensreduction}
\end{align}
such that the denominator is not zero.

\newcommand{\nonaafinebound}{\textsc{Non-Affine-Sensitivity}}
\setcounter{AlgoLine}{0}
\begin{algorithm}[th]
\caption{$\nonaafinebound(P,w,p, k,\Tau)$; see Lemma~\ref{lemma:svd}}\label{alg:tight1}
{\begin{minipage}{\textwidth}
\begin{tabbing}
\textbf{Input:} \quad\=A weighted set $(P,w)$ of $n$ points in $\mathbb{R}^d$, a point $p\in P$,
\\\> an integer $k\in\br{0,\cdots,d-1}$, and an error parameter $\Tau\in(0,1)$. \\
\textbf{Output:} An additive $\Tau$-approximation $s'$ to the sensitivity $s(p)$ of $p$.
\end{tabbing}\end{minipage}}

\nl $\mathbf{p} :=  \sqrt{w(p)}{p}^{T}$  ; $\mathbf{P}:=\begin{bmatrix}
        \sqrt{w(p_1)}{p}_1  \cdots  \sqrt{w(p_n)}{p}_n \\
            \end{bmatrix}^T$ \label{matP} 

\nl\If{$k=d-1$} {

\nl denote by $UDV^T$  the thin SVD of $\mathbf{P}.$	

\nl\Return $\norm{u}_2^2$ \quad\quad\tcp{where $u$ is the corresponding row in $U$ for $\mathbf{p}$ in $\mathbf{P}$.} \label{u_i line bound}	

}

\nl$\ell:= d-k$ \label{big p}

\nl$\displaystyle{\gamma := \frac{\sum_{i=1}^{\ell} \lambda_{d-i+1}( \mathbf{P}^T \mathbf{P}) }{\sum_{i=1}^{\ell} \lambda_{i}( \mathbf{P}^T \mathbf{P})}}$ \\ \tcp{ where $\lambda_{i}( \mathbf{P}^T \mathbf{P})$ is the $i$th eigenvalue of $ \mathbf{P}^T \mathbf{P}$ for every $i\in \br{1,\cdots,d}$.} \label{def gamma}

$X :=$ any $d\times {\ell}$ matrix whose columns are orthonormal (i.e., $X^TX=I$).

$s_{new} :=\infty$ ; $s_{old} := -\infty $

\While {$\displaystyle {s_{new} -  \frac {\Tau \gamma}{(1-\gamma)}\geq s_{old}}$}
{\label{while}

$\displaystyle {s_{new}:= \frac{\norm{{\mathbf{p}}X}^2_2}{\norm{{\mathbf{P}}X}^2_F}}$\label{snew}

$G := {\mathbf{p}^T \mathbf{p} } - s_{new} \cdot {\mathbf{P}^T \mathbf{P}}$\label{snewafter}

$X := \begin{bmatrix}
        x_1 \cdots  x_\ell \\
            \end{bmatrix}$  \quad\quad\tcp{a matrix that its cols are the eigenvectors corresponding to the $\ell$'th largest eigenvalues of $G$.}


}

$s' = s_{new}+\Tau$

\Return $s'$

\end{algorithm}
\begin{lemma}\label{lemma:svd}
Let $(P,w)$ be a weighted set of $n$ points in $\REAL^d$, $k\in \br{0,\cdots,d-1}$ be an integer, $\Tau \in (0,1)$ be an error parameter, and let $p\in P$. Let $\mathcal{S}_{d}$ denote the set of all non-affine $k$-subspaces in $\REAL^d$,  $\displaystyle{s(p) = \sup_{S'\in \mathcal{S}_{d}} \frac{w(p)D(p,S')}{\sqi w(q)D(q,S')}}$ denote the sensitivity of $p$ in the $(P,w,\mathcal{S}_{d},D)$ query space, where the denominator is not zero, and let $s'$ be the output of a call to $\nonaafinebound(P,w,p, k,\Tau)$; See Algorithm~\ref{alg:tight1}. Then the following holds according to $k$:
\begin{enumerate}[(i)]
\item if $k=d-1$ then  $s' = s(p)$.
\item if $k\in \br{0,\cdots,d-2}$ then  $s(p) \leq s' \leq s(p) + \Tau$.
\end{enumerate}
\end{lemma}

\begin{proof}
Let $\mathbf{P}$ and $\mathbf{p}$ be the matrix and row vector that are defined at Line~\ref{matP} of Algorithm~\ref{alg:tight1}.

\textbf{Proof of Claim (i) $k=d-1$}:
 Let $\mathbf{P}=UDV^T$ denote the thin Singular Value Decomposition of $\mathbf{P}$,
 and let $\mathbf{p}=uD V^T$ (where $u$ is the corresponding row in $U$ for the row $\mathbf{p}$ in $\mathbf{P}$ ). Let $x^*$ be the vector that maximizes $\frac{\abs{\mathbf{p}x}^2}{\norm{\mathbf{P}x}_2^2}$ over every $x\in\REAL^d$ such that $\norm{\mathbf{P}x}_2>0$ and $x^Tx=1$. It is well known (e.g.~\cite{yang2017weighted}) that,\\
\begin{align*}
s(p) =
&\frac{\abs{\mathbf{p}x^*}^2}{\norm{\mathbf{P}x^*}_2^2}
= \frac{\abs{uD V^Tx^*}^2}{\norm{UD V^Tx^*}_2^2}
= \frac{\abs{uDV^Tx^*}^2}{\norm{DV^Tx^*}_2^2}\leq \frac{\norm{u}_2^2\norm{DV^Tx^*}_2^2}{\norm{DV^Tx^*}_2^2}
= \norm{u}_2^2,
\end{align*}
where the first equality holds by~\eqref{sensreduction}, the second is by the definition of $u,U,D$ and $V$, the third is since the columns of $U$ are orthonormal, and the inequality holds by the Cauchy Schwarz inequality.
We now prove that this upper bound is tight. Indeed, substituting $x=VD^{-1}u^T$ (where $D^{-1}$ is the inverse matrix of $D$) attains this maximum as,
\begin{align*} 
\frac{\abs{\mathbf{p}x}^2}{\norm{\mathbf{P}x}_2^2} =
\frac{\abs{uDV^TVD^{-1}u^T}^2}{\norm {DV^TVD^{-1}u^T}_2^2}=
\frac{\abs{uu^T}^2}{\norm{u^T}_2^2} = \norm{u}_2^2.
\end{align*}
Thus, $s(p) =\norm{u}_2^2 $, which is the returned value of Algorithm~\ref{alg:tight1} for the case $k=d-1$.; See Line~\ref{u_i line bound}.
\newcommand{\X}{\displaystyle{ X^*}}\\

\textbf{Proof of Claim (ii) $k\in\br{0,\cdots,d-2}$}: Let $\X$ be the Matrix that maximizes $\frac{\norm{\mathbf{p}X}_2^2}{\norm{\mathbf{P}X}_F^2}$ over every $X\in\REAL^{d\times (d-k)}$ such that $\norm{\mathbf{P}X}_F>0$ and $X^TX=I$. We have,
\begin{align*}
s(p) =
\frac{\norm{\mathbf{p}\X}_2^2 }{ \norm{\mathbf{P}{\X}}_F^2}=
{\frac{Tr({\X}^T\mathbf{p}^T\mathbf{p}{\X})}{Tr({\X}^T \mathbf{P}^T \mathbf{P} {\X})}}.
\end{align*}

Let $j_{\max}>1$ be the number of iterations that are executed in the ``while" loop of Algorithm~\ref{alg:tight1} until it stops, and let $s_j$ be the value of $s_{new}$ during the execution of Line~\ref{snewafter} in the $j$th iteration for every $j\in \br{1,\cdots,j_{\max}}$.
The while loop of Algorithm~\ref{alg:tight1} is the same while loop of Algorithm~2 in~\cite{zhang2010fast}, where the main difference is the stopping criterion.
  In \cite{zhang2010fast} it was proven that $s_{new}$ (in the ``while`` loop of Algorithm~\ref{alg:tight1}) converges to the global supremum of $$ \displaystyle{\frac{Tr(X^T\mathbf{p}^T\mathbf{p}X)}{Tr(X^T {\mathbf{P}^T\mathbf{P}} X)}}.$$
 Moreover $s_j \leq s(p)$ for every $j\in \br{1,\cdots,j_{\max}}$.

  Let $\gamma$ be defined as at Line~\ref{def gamma} of Algorithm~\ref{alg:tight1}. Based on Theorem 5.1 in \cite{zhang2010fast}, for any integer $j\geq 1$, we have that,
\begin{align*}
s(p) - s_j &\leq (1-\gamma)(s(p)  - s_{j-1}) .
\end{align*}
Rearranging yields,
\begin{align}
s(p) \leq \frac{s_j - s_{j-1}(1-\gamma)}{\gamma}. \label{th5.1ofzhang}
\end{align}
By the stopping criterion (in Line~\ref{while} of Algorithm~\ref{alg:tight1}) we have that after the last iteration
\begin{align}\left(s_{j_{\max}} -  \frac { \eps \gamma}{1-\gamma}\right)< s_{{j_{\max}-1}}. \label{sc-1 bound}\end{align}
By combining~\eqref{sc-1 bound} with~\eqref{th5.1ofzhang} we obtain an upper bound on $s(p)$, as:
$$ s(p) \leq  \frac{s_{j_{\max}}  -  (s_{j_{\max}}  -  \frac {  \Tau \gamma}{1-\gamma})  (1-\gamma)}{\gamma} = \frac{s_{j_{\max}}  - s_{j_{\max}}  + \gamma s_{j_{\max}}  +  \Tau\gamma}{\gamma}=s_{j_{\max}} + \Tau. $$
By the above inequality and since $s_{j_{\max}} \leq s(p)$, we have
\begin{align*}
s(p) \leq  s_{j_{\max}}  +  \Tau \leq s(p)+ \Tau.
\end{align*}
We conclude that the returned value of the algorithm $s' =s_{j_{\max}} +  \Tau$ satisfies Claim~(ii) as,
$$s(p) \leq  s' \leq s(p)+ \Tau.$$
\end{proof}
\section{Reduction from Affine to Non-Affine Subspace} \label{reducsec}
In this section we use the two integers $n,d\geq1$,  the number $z\geq1$ and $\Tau\in(0,  \frac{1}{2^{z+1} +2} ]$ as the additive error of our sensitivity bounds, which is polynomial in $1/n$ in our experimental results. We let $(P,w)$ be a weighted set of $n$ points in $\REAL^d$, $\psi= \left(\frac{\Tau}{z}\right)^z$,  $\displaystyle r = 1+ \max_{p\in P}\frac{D_z(p,0_d)}{\psi\Tau^2}$, and $\displaystyle{e_{d+1}=(0, \cdots,0,1)\in\REAL^{d+1}}$.

For every $q\in P$ we let $q' = (q\mid r) $ and $P' = \br{ q' \mid q\in P }$. The set of all affine $k$-subspaces in $\REAL^d$ is denoted by $\mathcal{S}^{A}_{d}$.
  For every affine $k$-subspace $S\in \mathcal{S}^{A}_{d}$, we define $\displaystyle{S'' := \{ (x\mid r)\mid x\in S\}}$ to be its corresponding affine $k$-subspace in $\REAL^{d+1}$, and $S'$ to be the corresponding (non-affine) $(k + 1)$-subspace of $\REAL^{d+1}$ that is spanned by $S''$; see Figs.~\ref{fig:big} and~\ref{fig:small}. Finally let $\displaystyle{\mathcal{S}_{d+1}}$ denote the union over all $(k+1)$ non-affine subspaces of $\REAL^{d+1}$.

\begin{figure}
  \centering
  \includegraphics[width=0.8\columnwidth]{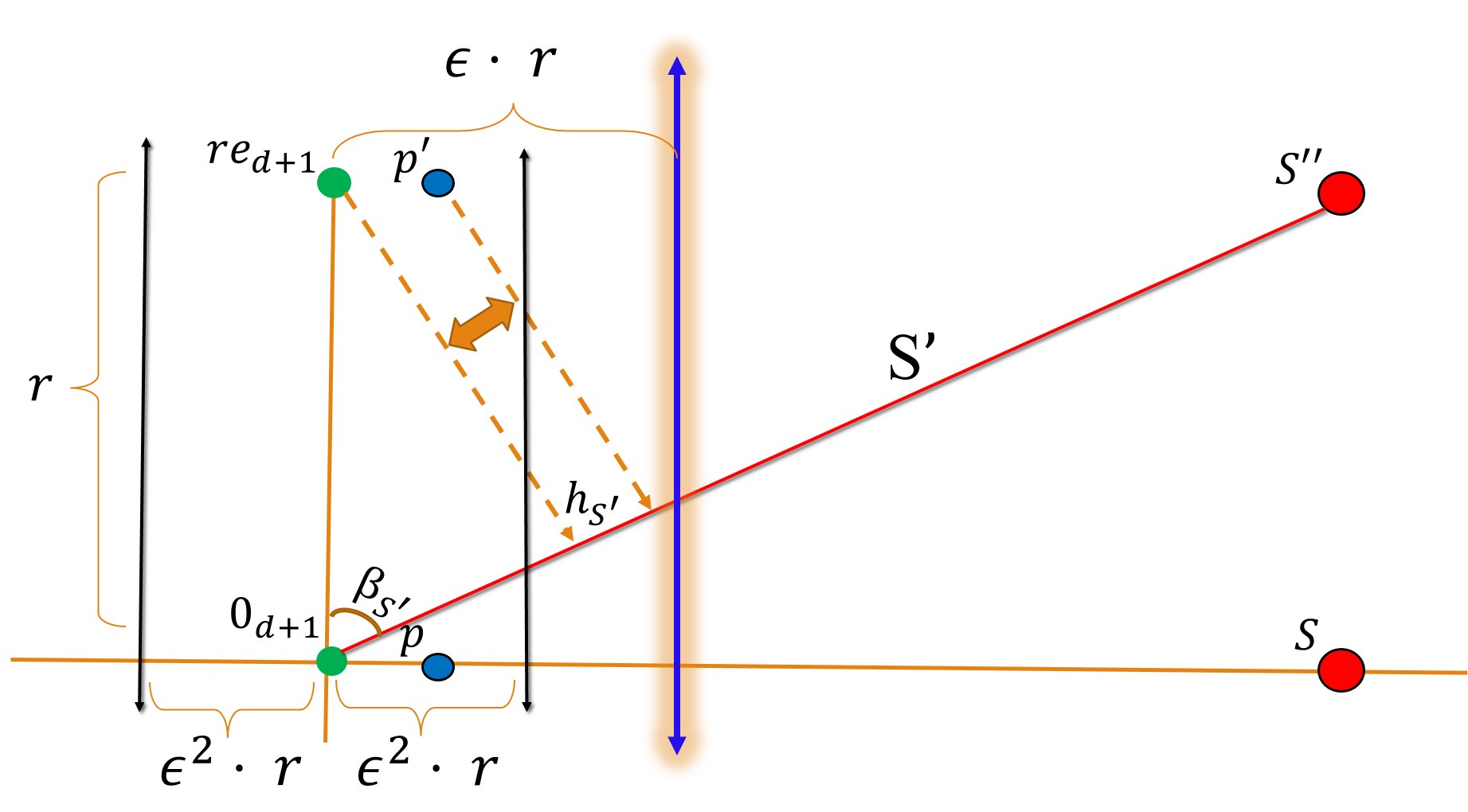}
  \caption{ Illustration of Claim (i) at Lemma~\ref{theorem:j-affine}.  This figure illustrates the case where $d=1$, $z=1$, and $k=0$, $p$ is a point in $\REAL$ such that $D_1(p,0_d) = \dist(p,0_d) = \abs{p-0_{d}} \leq \Tau^2 r$, and $p'=(p\mid r)\in \REAL^2$, $S$ is an affine $0$-subspace of $\REAL$ (a point in $\REAL$ for this case) such that $D_1(0_d,S)  = \dist(0_d,S) = \abs{S-0_{d}}\geq \Tau r$ , $S''$ is the corresponding affine $0$-subspace of $\REAL^2$, and $S'$ is the (non-affine) $1$-subspace (i.e., line)  of $\REAL^2$ that passes through $S''$. This illustration aims to show that the distance from $p$ to $S$ (i.e, $D_1(p,S)$) is approximately equal to the distance from $0_d$ to $S$ ($D_1(p,S)$), and that $D_1(p',S')$ is approximately equal to $D_1(re_{2},S' )$.  }
  \label{fig:big}
\end{figure}

\begin{figure}
  \centering
  \includegraphics[width=0.8\columnwidth]{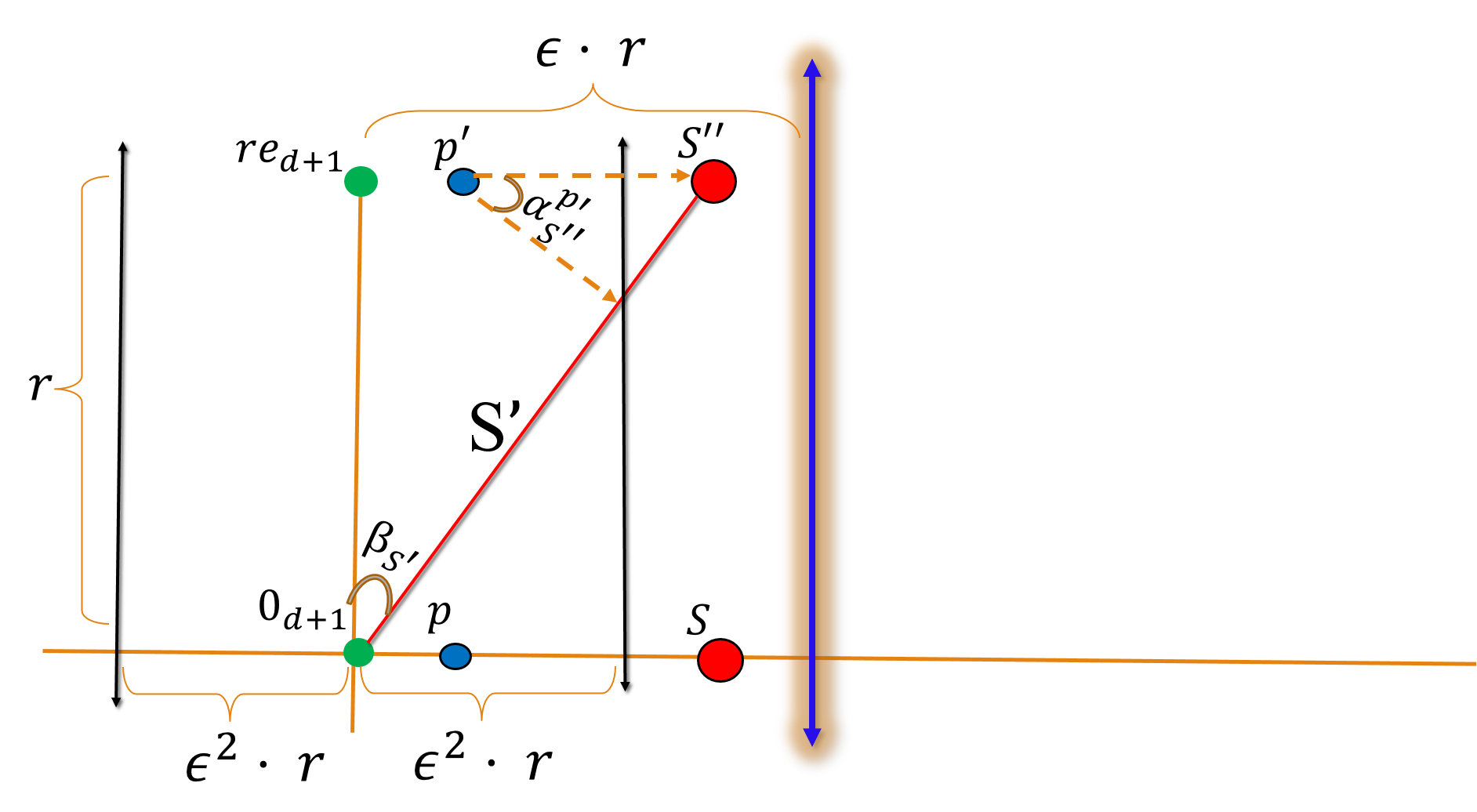}
  \caption{Illustration of Claim (ii) at Lemma~\ref{theorem:j-affine}. Same as Figure~\ref{fig:big} we have $d=1$, $z=1$, $k=0$, $p$ is a point in $\REAL$ such that $D_1(p,0_d) = \dist(p,0_d) = \abs{p-0_{d}} \leq \Tau^2 r$, $p'=(p\mid r)\in \REAL^2$, $S$ is an affine $0$-subspace of $\REAL$ (a point in $\REAL$ for this case) such that $D_1(0_d,S)  = \dist(0_d,S) = \abs{S-0_{d}}< \Tau r$ , $S''$ is the corresponding affine $0$-subspace of $\REAL^2$, and $S'$ is the (non-affie) $1$-subspace (i.e., line)  of $\REAL^2$ that passes through $S''$. This illustration aims to show that the distance from $p$ to $S$ (i.e, $D_1(p,S)$) is approximately equal to the distance from $p'$ to $S'$ (i.e., $D_1(p',S')$).}
\label{fig:small}
\end{figure}
\begin{lemma}\label{theorem:j-affine}
Let $S\in \mathcal{S}^{A}_{d}$ be an affine $k$-subspace of $\REAL^d$. For every $p \in P$ and its corresponding point $p'=(p\mid r)$, Claims (i)--(ii) hold as follows:
\begin{enumerate}[(i)]
\item if $D_z(0_d,S) \geq \Tau r$ then
\begin{align}
&| D_z(0_{d},S) - D_z(p,S) | \leq 2\Tau \cdot D_z(0_{d},S)\label{oo1}
\intertext{and}
&| D_z(re_{d+1},S') - D_z(p',S') |\leq  (2^z+1 )\Tau \cdot D_z(re_{d+1},S').\label{oo2}
\end{align}
\item if $D_z(0_d,S) < \Tau r$, then 
\begin{align}
&D_z(p',S') \leq D_z(p,S) \leq (1+\constalaa \Tau)D_z(p',S').\label{oo3}
\end{align}
\end{enumerate}
\end{lemma}

\begin{proof}
By~\cite[Lemma~2.1]{Feldman:2012:DRW:2095116.2095222}, for every $x,y\in\REAL^d$ we have
\begin{align}
\abs{D_z(x,S) - D_z(y,S)} \leq \frac{D_z(x,y)}{\psi} + \Tau D_z(x,S).\label{knownfact}
\end{align}
Let $p\in P$ and let $p'=(p\mid r)$.

\textbf{Proof of Claim (i): } Inequality~\eqref{oo1} holds since
\begin{align*}
|D_z(p,S) - D_z(0_d,S)|&\leq \frac{D_z(p,0_d)}{\psi} + \Tau  D_z(0_d,S) \leq \Tau^2r + \Tau  D_z(0_d,S) \leq 2\Tau D_z(0_d,S),
\end{align*}
where the first inequality holds by substituting $x=0_d$ and $y=p$ in~\eqref{knownfact}, the second is by the definition of $r$, and the last inequality holds by the assumption of Claim~(i).

Let $h_{S''}=\mathrm{proj}(re_{d+1},S'')$. To prove~\eqref{oo2}, we observe that
\begin{align}
\dist(re_{d+1}, h_{S''}) = \dist(re_{d+1}, S'') = \dist(0_d,S)=D^{1/z}_z(0_d,S) \geq (\Tau r)^{\frac{1}{z}}, \label{dist bound}
\end{align}
where the first equality holds by the definition of $h_{S''}$, the second holds by the definition of $S''$ and $e_{d+1}$, the third holds by the definition of $D_z$, and the inequality holds by taking the power of $1/z$ of each side of the assumption $D_z(0_d,S)\geq \eps r$.
In addition let $h_{S'}=\mathrm{proj}(re_{d+1},S')$, and let $\beta_{S'} \in [0, \pi/2)$ denote the angle $\angle(re_{d+1}, 0_{d+1}, h_{S'})$. Hence,
\begin{align}
& \dist(re_{d+1},h_{S'}) = \sin{\beta_{S'}} \cdot \dist(re_{d+1}, 0_{d+1}) = \sin{\beta_{S'}} \cdot r.\label{dist(re1,projjre)}
\end{align}
Now we compute a lower bound on $\sin{\beta_{S'}}$:
\begin{align}
\sin{\beta_{S'}} & =  \frac{\dist(re_{d+1},h_{S''} )} {\dist(0_{d+1} , h_{S''})} \label{eq1}
\\&\geq  \frac{\dist(re_{d+1}, h_{S''})} { \dist( 0_{d+1},re_{d+1})+  \dist(re_{d+1}, h_{S''})  }\label{eq2}
\\&\geq  \frac{(\Tau r)^{\frac{1}{z}}} { \dist( 0_{d+1},re_{d+1})+  (\Tau r)^{\frac{1}{z}} }\label{eq3}
\\&=  \frac{(\Tau r)^{\frac{1}{z}}}{r+ (\Tau r)^{\frac{1}{z}}}
=\frac{\Tau^{\frac{1}{z}}} {\frac{r}{r^{\frac{1}{z}}} +\Tau^{\frac{1}{z}}}
=\frac{\Tau^{\frac{1}{z}}} {r^{\frac{z-1}{z}} +\Tau^{\frac{1}{z}}}\label{minval}\\
&\geq \frac{\Tau^{\frac{1}{z}}}{2r^\frac{z-1}{z}} \label{minval1}
\end{align}
where~\eqref{eq2} holds by the triangle inequality, \eqref{eq3} holds by~\eqref{dist bound},
and~\eqref{minval} holds since $\dist(0_{d+1},re_{d+1})= r$.
 Equation \eqref{minval1} holds since $r^{\frac{z-1}{z}}\geq\Tau^{\frac{1}{z}}$ for $z\geq1$, $r>1$ and $\Tau<1$.
 Plugging~\eqref{minval1} in~\eqref{dist(re1,projjre)} yields,
\begin{align}
& \dist(re_{d+1},h_{S'}) \geq  \frac{\Tau^{\frac{1}{z}}}{2r^\frac{z-1}{z}}\cdot r = \frac{(\Tau r)^{1/z}}{2}.\label{distbound}
\end{align}
By~\eqref{distbound} and the definition of $D_z$,
\begin{align}
& D_z(re_{d+1},S') = D_z(re_{d+1},\mathrm{proj}(re_{d+1},S'))=D_z(re_{d+1},h_{S'})\geq\bigg(\frac{(\Tau r)^{1/z}}{2}\bigg)^z
=\frac{\Tau r}{2^z}.\label{goodone}
\end{align}
To complete the proof of Claim (i),
\begin{align}
|D_z(p',S') - D_z(re_{d+1},S')| &\leq \frac{D_z(p',re_{d+1})}{\psi} + \Tau D_z(re_{d+1},S') \label{nn1} \\
& = \frac{D_z(p ,0_{d})}{\psi} + \Tau D_z(re_{d+1},S')\label{nn2} \\
& \leq\Tau^2 r + \Tau  D_z(re_{d+1},S') \label{nn3} \\
& \leq   (2^z+1 )\Tau D_z(re_{d+1},S'), \label{diffbound2}
\end{align}
where~\eqref{nn1} holds by substituting $x=p'$, $y=re_{d+1}$ and $S=S'$ in~\eqref{knownfact},~\eqref{nn2} is by the definition of $e_{d+1}$ and $p'$,~\eqref{nn3} is by the definition of $r$, and~\eqref{diffbound2} holds by~\eqref{goodone}.
%
\newcommand{\h}{ h_{S''}^{p'}}
\newcommand{\alphamy}{\alpha_{S''}^{p'}}

\textbf{Proof of Claim (ii):} In this case $D_z(0_d,S)< \Tau r$. Let $h_{S''}^{p'}=\mathrm{proj}(p',S'')=(\mathrm{proj}(p,S)^T\mid r)^T$ and $\alphamy \in [0, \pi/2)$ denote the angle $\angle(\mathrm{proj}(p', S'), p',\h)$. Hence,
\begin{align}\label{p,S}
\dist(p,S) = \dist(p',S'')=\frac{\dist(p',S') }{\cos{\alphamy}}  .
\end{align}
We have that $\cos{\alphamy} \leq 1$, thus $\dist(p',S') \leq \dist(p,S)$. Observing that $$\alphamy=\angle(re_{d+1}, 0_{d+1}, h_{S''}^{p'}) =\angle(re_{d+1}, 0_{d+1}, h_{S'}) = \beta_{S'}$$ yields, 
\begin{align}
\cos{\alphamy} =  \frac{re_{d+1}^T \h}{\norm{re_{d+1}}_2\norm{\h}_2} = \frac{(0,\cdots,0,r) (\mathrm{proj}(p,S)^T\mid r)^T}{\norm{(0,\cdots,0,r)}_2\norm{\h}_2} = \frac{r^2}{r\norm{\h}_2} = \frac{r}{\norm{\h}_2}.\label{cosal}
\end{align}
By the triangle inequality,
\begin{align}\label{normh}
\norm{\h}_2 \leq \norm{0_{d+1} -re_{d+1}}_2 + \norm{re_{d+1} - p'}_2 + \norm{p' - \h}_2. 
\end{align}
We also have $\norm{p'-\h}_2=\dist(p,S)$  and by the triangle inequality $\dist(p,S)\leq \norm{p}_2 + \dist(0_d,S)$. The bound on $\norm{p}_2$ is,
\begin{align}
\norm{p}_2 = (D(p,0_d))^{1/z} \leq (\Tau^2\psi r)^{1/z}\leq \Tau r. \label{normp}
\end{align}
Thus,
\begin{align}\label{boundph}
\norm{p'-\h}_2=\dist(p,S)\leq \norm{p}_2 + \dist(0_d,S) \leq 2\Tau r,
\end{align}
where the last inequality is by~\eqref{normp} and the assumption of Claim (ii). And by plugging~\eqref{boundph} in~\eqref{normh} and using the facts that $\norm{0_{d+1} -re_{d+1}}_2=r$ and $\norm{re_{d+1} - p'}_2 = \norm{p}_2\leq \Tau r$ we get that,
\begin{align*}
\norm{\h}_2 \leq r +\norm{p}_2 + 2\Tau r \leq (1+3\Tau) r.
\end{align*}
Together with~\eqref{cosal} this yields a lower bound on $\cos{\alphamy}$ as,
\begin{align*}
\cos{\alphamy}= \frac{r}{\norm{\h}_2} \geq  \frac{1}{ 1+3\Tau }.
\end{align*}
Now we obtain an upper bound for~\eqref{p,S} as,
\begin{align*}
\dist(p,S) = \frac{\dist(p',S') }{\cos{\alphamy}}  \leq (1+3 \Tau) \dist(p',S') .
\end{align*}

Thus we have that $ \dist(p',S') \leq \dist(p,S)  \leq(1+3 \Tau) \dist(p',S')$. By the definition of $D_z$ we have
\begin{align}
D_z(p',S') \leq D_z(p,S) \leq (1+3\Tau)^z D_z(p',S'). \label{plughere}
\end{align}
This proves Claim (ii) for the case $z=1$. Otherwise, by the Bernoulli's inequality we have that $(1+x)^y \leq 1+ \frac{xy}{1-(y-1)x}$ for every $x\in [-1,\frac{1}{y-1})$ and $y>1$. Thus substituting $y=z,x=3\Tau$ yields,
\begin{align*}
(1+3\Tau)^z D_z(p',S')  \leq \left(1 + \frac{3z\Tau}{1-(z-1)3\Tau}\right) D_z(p',S').
\end{align*}
Observe that since $\Tau \leq  \frac{1}{2^{z+1} +2} \leq \frac{1}{6(z-1)}$ we have that $1-(z-1)3\Tau \geq 1-(z-1)\frac{3}{6(z-1)} = 1/2$.  Hence,
\begin{align*}
 (1+3\Tau)^z D_z(p',S')  \leq \left(1 + \frac{3z\Tau}{1-(z-1)3\Tau}\right) D_z(p',S') \leq (1 +\constalaa \Tau) D_z(p',S').
\end{align*}
Plugging the last inequality in~\eqref{plughere} proves Claim (ii) for $z>1$ as
$$D_z(p',S') \leq D_z(p,S) \leq  (1 +\constalaa \Tau) D_z(p',S').$$
\end{proof}
\begin{lemma}\label{affinetononaffinecor}
For every $\displaystyle{p\in P}$ and its corresponding point $p'=(p\mid r)$ we have
\begin{align*}
\left|\sup_{S\in \mathcal{S}^{A}_{d}}\frac{w(p)\cdot D_z(p,S)}{\sqi w(q) D_z(q,S)} - \sup_{S'\in \mathcal{S}_{d+1}}\frac{w(p)D_z(p',S')}{\sqi w(q) D_z(q',S')}\right| \leq 16\Tau(2^z +1) \cdot \sup_{S'\in \mathcal{S}_{d+1}}\frac{w(p)D_z(p',S')}{\sqi w(q) D_z(q',S')}.
\end{align*}
\end{lemma}

\begin{proof}
For every (non-affine) subspace $S'\in \mathcal{S}_{d+1}$ we denote $h_{S'}=\mathrm{proj}(re_{d+1},S')$ and $\beta_{S'} \in [0, \pi/2]$ denote the angle $\angle(re_{d+1}, 0_{d+1}, h_{S'})$. The following observation is from Figure~\ref{fig:big}:\\
\begin{observation} \label{obs:1} Let $S'$ be a non-affine $(k+1)$-subspace of $\REAL^{d+1}$ such that $\beta_{S'} < \pi/2$. Then the intersection of $S'$ with the hyperplane $\br{(x\mid r) \mid x \in \REAL^d}$ is the affine $k$-subspace $S'' \subset S'$ of $\REAL^{d+1}$ such that $S'$ is the linear span of $S''$, and the $(d+1)$th (last) coordinate of every point $x\in S''$ is $r$. Moreover there is an affine $k$-subspace $S$ of $\REAL^d$ such that $S'' = \br{(x \mid r) \mid x \in S}$. Hence,
\begin{align}
\sin{\beta_{S'}} = \frac{\dist(re_{d+1} , S'')}{ \dist(0_{d+1} , S'') }= \frac{\dist(re_{d+1} , S'')}{\sqrt{\dist^2(re_{d+1} , S'') + \dist^2(re_{d+1},0_{d+1})^2 } } = \frac{\dist(0_{d} , S)}{\sqrt{ \dist^2(0_{d} , S) +r^2}}. \label{sineq}
\end{align}
\end{observation}
The above observation will be used through the proof. Let $c_0 =\frac{(\Tau r)^{1/z}}{ \sqrt{(\Tau r)^{2/z} + r^2}}$. We partition the query set $\mathcal{S}_{d+1}$ into two disjoint subsets:
\begin{enumerate}[(i)]
\item $\displaystyle {Q'_{0} = \br{S' \in \mathcal{S}_{d+1} \mid  \sin{\beta_{S'}} \geq  c_0  }}$, and
\item $\displaystyle {Q'_{1} = \br{S' \in \mathcal{S}_{d+1}\mid 0\leq\sin{\beta_{S'}} < c_0  }}$.
\end{enumerate}
Similarly, we partition $\mathcal{S}^{A}_{d}$ into :
\begin{enumerate}[(i)]
\item $\displaystyle{Q_{0} =\br{S \in \mathcal{S}^{A}_{d}  \mid  D_z(0_d,S) \geq \Tau r}}$, and
\item $\displaystyle{Q_{1} =\br{S \in \mathcal{S}^{A}_{d}  \mid  D_z(0_d,S) < \Tau r }}$.
\end{enumerate}
Let $c_1 = 2^z +1$. We first proof the following pair of claims: \\
\textbf{Claim (i).} For every $p\in P$ we have
\begin{align*}
\left|\sup_{S\in Q_0}\frac{w(p)\cdot D_z(p,S)}{\sqi w(q) D_z(q,S)} - \sup_{S'\in Q_0'}\frac{w(p)D_z(p',S')}{\sqi w(q) D_z(q',S')}\right| \leq 16c_1\Tau \sup_{S'\in Q_0'}\frac{w(p)D_z(p',S')}{\sqi w(q) D_z(q',S')}.
\end{align*}
\textbf{Claim (ii).} For every $p\in P$ we have
\begin{align*}
\left|\sup_{S\in Q_1}\frac{w(p)\cdot D_z(p,S)}{\sqi w(q) D_z(q,S)} - \sup_{S'\in Q_1'}\frac{w(p)D_z(p',S')}{\sqi w(q) D_z(q',S')}\right| \leq 16c_1\Tau \sup_{S'\in Q_1'}\frac{w(p)D_z(p',S')}{\sqi w(q) D_z(q',S')}.
\end{align*}
Since $\mathcal{S}^{A}_{d} = Q_0 \cup Q_1$ and $\mathcal{S}_{d+1} = Q_0' \cup Q_1'$, combining both claims poofs the lemma.

\textbf{Proof of Claim (i).} By~\eqref{oo1} and since $c_1\geq 2$, for every pair $S\in Q_0$ and $p\in P$ we have
$$\displaystyle {(1-c_1\Tau)D_z(0_{d},S) \leq D_z(p,S) \leq(1+c_1\Tau) D_z(0_{d},S)}.$$
Hence,
\begin{align*}
\frac{w(p) (1-c_1\Tau)D_z(0_d,S) }{\sqi w(q)  (1+c_1\Tau) D_z(0_d,S) } \leq \frac{w(p)D_z(p,S)}{\sqi w(q) D_z(q,S)}  \leq \frac{w(p)(1+ c_1\Tau)D(0_d,S) }{\sqi w(q) (1- c_1\Tau)D_z(0_d,S) }.
\end{align*}
Since the above inequality holds for every $S \in Q_0$ we get that,
\begin{align}
\frac{(1- c_1\Tau)}{(1+c_1\Tau)}\frac{w(p)}{\sqi w(q) } \leq \sup_{S\in Q_0}\frac{w(p)D_z(p,S)}{\sqi w(q) D_z(q,S)}  \leq \frac{(1+c_1\Tau)}{(1- c_1\Tau)}\frac{w(p)}{ \sqi w(q) }. \label{beforebound}
\end{align}

Let $p'\in P'$. We now prove that for every $S' \in Q_0'$
\begin{align}
(1-c_1\Tau)D_z(re_{d+1},S') \leq D_z(p',S') \leq(1+c_1\Tau) D_z(re_{d+1},S') \label{bigsin}
\end{align}
by case analysis: first for $\sin\beta_{S'} \in [c_0,1)$ and then for $\sin\beta_{S'}=1$.
 If $S' \in  \br{ S' \in \mathcal{S}_{d+1} \mid  \sin{\beta_{S'}}\in [c_0,1) }$ then we have that $\beta_{S'} < \pi/2$. Hence, by  Observation~\ref{obs:1} there is an affine $k$-subspace $S$ (of $\REAL^d$) such that
$$\displaystyle{\sin{\beta_{S'}} = \frac{\dist(0_{d} , S)}{\sqrt{ \dist^2(0_{d} , S) + r^2}}}.$$
Combining this equality with the fact that $\sin{\beta_{S'}} \geq  c_0 =\frac{(\Tau r)^{1/z}}{\sqrt{(\Tau r)^{2/z} + r^2}}$ yields that $\dist(0_{d} , S) \geq (\Tau r)^{1/z}$. Taking the power of $z$ from both sides yields $D_z(0_d,S) \geq \Tau r$. Using this in~\eqref{oo2} yields that~\eqref{bigsin} holds for the case $\sin\beta_{S'} \in [c_0,1)$.

If $\sin \beta_{S'}= 1$ then we have that $\beta_{S'} = \pi/2$. This implies that $h_{S'} = 0_{d+1}$. Hence,
\begin{align}
\dist(re_{d+1},S') = \dist(re_{d+1},h_{S'}) =\dist(re_{d+1},0_{d+1}) = r \label{eqr}.
\end{align}
Hence, for every $S' \in \br{ S' \in\mathcal{S}_{d+1} \mid \sin{\beta_{S'}}  = 1 }$ we have
\begin{align}
\abs{D_z(re_{d+1},S') - D_z(p',S')}
&\leq \frac{D_z(re_{d+1},p')}{\psi} + \Tau D_z(re_{d+1},S')\label{n0}   \\
&\leq \frac{D_z(0_{d},p)}{\psi} + \Tau D_z(re_{d+1},S')\label{n1} \\
&\leq \frac{r\psi\Tau^2}{\psi} +  \Tau D_z(re_{d+1},S')\label{n2}\\
&= \Tau^2 r + \Tau  D_z(re_{d+1},S') \nonumber \\
&\leq  2\Tau  D_z(re_{d+1},S') \label{n4} \\
& \leq  c_1\Tau  D_z(re_{d+1},S')\label{n5} ,
\end{align}
where~\eqref{n0} holds by substituting $x=re_{d+1}$, $y=p'$ and $S=S'$ in~\eqref{knownfact},~\eqref{n1} is by the definition of $r$ and $p'$,~\eqref{n2} holds by the definiton of $\psi$ and $r$, ~\eqref{n4} holds by~\eqref{eqr}, and ~\eqref{n5} holds since $c_1 > 2$. This proves~\eqref{bigsin} also for the case that $\sin \beta_{S'}=1$.  Hence,~\eqref{bigsin} holds for every $S' \in Q_0'$.

By~\eqref{bigsin} we get that for every $p'\in P'$
\begin{align}
\frac{(1-c_1\Tau)}{(1+ c_1\Tau)}\frac{w(p)  }{\sqi w(q) } \leq \sup_{S'\in Q_0'}\frac{w(p) D_z(p',S') }{\sqi w(q)D_z(q',S') }  \leq \frac{(1+c_1\Tau)}{(1- c_1\Tau)}\frac{w(p)}{\sqi w(q)  } .\label{doneHtag}
\end{align}


Integrating~\eqref{doneHtag} with~\eqref{beforebound} yields
\begin{align} 
\frac{(1- c_1\Tau)^2}{(1+ c_1\Tau)^2}\sup_{S'\in Q_0'} \frac{w(p)D_z(p',S')}{\sqi w(q) D_z(q',S')} &  \leq
\frac{(1- c_1\Tau)}{(1+ c_1\Tau)}\frac{w(p)}{\sqi w(q) } \label{re1}
\\ &  \leq \sup_{S\in Q_0}\frac{w(p)D_z(p,S)}{\sqi w(q) D_z(q,S)} \label{re2} \\
&\leq \frac{(1+c_1\Tau)}{(1- c_1\Tau)}\frac{w(p)}{ \sqi w(q)  }\label{re3} \\
&\leq \frac{(1+ c_1\Tau)^2}{(1- c_1\Tau)^2}\sup_{S'\in Q_0'}\frac{w(p)D_z(p',S')}{\sqi w(q) D_z(q',S')},  \label{re4}
\end{align}
where~\eqref{re1} holds by multiplying the right hand side of~\eqref{doneHtag} by $\frac{(1- c_1\Tau)^2}{(1+ c_1\Tau)^2}$,~\eqref{re2} and~\eqref{re3} hold by~\eqref{beforebound}, and~\eqref{re4} holds by multiplying the left hand side of~\eqref{doneHtag} by $\frac{(1+ c_1\Tau)^2}{(1- c_1\Tau)^2}$.
We have
\begin{align}
\frac{(1-c_1\Tau)^2}{(1+ c_1\Tau)^2} &= \frac{1+c_1^2\Tau^2-2c_1\Tau}{(1+c_1\Tau)^2} =\frac{1+2c_1\Tau+c_1^2\Tau^2-4c_1\Tau}{(1+c_1\Tau)^2}\label{vav}
\\&=1-\frac{4c_1\Tau}{(1+c_1\Tau)^2} \geq  1-4c_1\Tau \geq 1-16c_1\Tau, \nonumber
\end{align}
and
\begin{align}
\frac{(1+c_1\Tau)^2}{(1 -  c_1\Tau)^2}&=\frac{1+c_1^2\Tau^2+2c_1\Tau}{(1 -c_1\Tau)^2}=\frac{1 -2c_1\Tau+c_1^2\Tau^2 + 4c_1\Tau}{(1 -c_1\Tau)^2}\label{vav1}  \\&=1+\frac{4c_1\Tau}{(1 -c_1\Tau)^2} \leq 1+\frac{4c_1\Tau}{(1 -\frac {c_1}{2 c_1} )^2} = 1+\frac{4c_1\Tau}{ \frac {1}{4} } = 1+16c_1 \Tau, \nonumber
\end{align}
where the inequality in~\eqref{vav1} holds since $\Tau \in \left(0,\frac{1}{2(2^z +1)}\right)$. By plugging~\eqref{vav} and~\eqref{vav1} in \eqref{re4} we get
\begin{align*}
(1- 16c_1\Tau)\sup_{S'\in Q_0'}\frac{w(p)D_z(p',S')}{\sqi w(q) D_z(q',S')} \leq  \sup_{S\in Q_0} \frac{w(p)D_z(p,S)}{\sqi w(q) D_z(q,S)}  \leq
(1+16c_1\Tau)\sup_{S'\in Q_0'}\frac{w(p)D_z(p',S')}{\sqi w(q) D_z(q',S')}.
\end{align*}
This proves Claim (i) as
$$\left|\sup_{S\in Q_0}\frac{w(p)\cdot D_z(p,S)}{\sqi w(q) D_z(q,S)} - \sup_{S'\in Q_0'}\frac{w(p)D_z(p',S')}{\sqi w(q) D_z(q',S')}\right| \leq 16c_1\Tau \sup_{S'\in Q_0'}\frac{w(p)D_z(p',S')}{\sqi w(q) D_z(q',S')}.$$

\textbf{Proof of Claim (ii).}
Let $S'\in Q'_1$. We have that $\beta_{S'}<\pi/2$. Hence, by Observation~\ref{obs:1} there is an affine $k$-subspace $S$ (of $\REAL^d$) such that
\begin{align}
\displaystyle{\sin{\beta_{S'}} = \frac{\dist(0_{d} , S)}{\sqrt{ \dist^2(0_{d} , S) + r^2}}}. \label{sinbetaa}
\end{align}
By the definition of $S'$ we have that $\sin{\beta_{S'}} <c_0 = \frac{(\Tau r)^{1/z}}{ \sqrt{(\Tau r)^{2/z} +  r^2}}$. Combining this with~\eqref{sinbetaa} yields that $\dist(0_{d} , S) < (\Tau r)^{1/z}$. Taking the power of $z$ from both sides yields that $D_z(0_d,S) < \Tau r$.
Using this is~\eqref{oo3} yields that for every $p' \in P$ and $S'\in Q_1'$ we have 
\begin{align*}
\frac{1}{(1+\constalaa\Tau)}\frac{ w(p)D_z(p',S')}{\sqi w(q) D_z(q',S')} \leq  \frac{w(p) D_z(p,S)}{\sqi w(q) D_z(q,S)} \leq  (1+\constalaa\Tau)\frac{w(p)  D_z(p',S')}{\sqi w(q) D_z(q',S') }.
\end{align*}
By combining the last inequality with the fact that
\begin{align*}
1-\constalaa\Tau\leq1-\frac{\constalaa\Tau}{1+\constalaa\Tau}=\frac{1+\constalaa\Tau-\constalaa\Tau}{1+\constalaa\Tau}=\frac{1}{1+\constalaa\Tau},
\end{align*}
we obtain that
\begin{align*}
&(1-\constalaa\Tau)\frac{ w(p)D_z(p',S')}{\sqi w(q) D_z(q',S')} \leq  \frac{w(p) D_z(p,S)}{\sqi w(q) D_z(q,S)} \leq  (1+\constalaa\Tau)\frac{w(p)  D_z(p',S')}{\sqi w(q) D_z(q',S')}.
\end{align*}
Since the above inequalities hold for every $S' \in Q_1'$ and $16c_1=16(2^z+1)>\constalaa$, this proves Claim~(ii) as
\begin{align*}
(1-16c_1 \Tau)\sup_{S' \in Q_1'}\frac{w(p)D_z(p',S')}{\sqi w(q) D_z(q',S')} \leq\sup_{S \in Q_1} \frac{w(p)D_z(p,S)}{\sqi w(q) D_z(q,S)}  \leq
(1+16c_1\Tau)\sup_{S' \in Q_1'}\frac{w(p)D_z(p',S')}{\sqi w(q) D_z(q',S')},
\end{align*}
i.e.,
\begin{align*}
\bigg|\sup_{S'\in Q_1'}\frac{w(p)D_z(p',S')}{\sqi w(q) D_z(q',S')} -  \sup_{S\in Q_1}\frac{w(p)D_z(p,S)}{\sqi w(q) D_z(q,S)}\bigg|  \leq 16\Tau c_1 \cdot \sup_{S'\in Q_1'}\frac{w(p)D_z(p',S')}{\sqi w(q) D_z(q',S')}.
\end{align*}
\end{proof}

\section{Sensitivity of affine $k$-subspaces}
\newcommand{\afinebound}{\textsc{Affine-Sensitivity}}
\setcounter{AlgoLine}{0}
\begin{algorithm}[th]
\caption{$\afinebound(P,w,p, k,\Tau)$; see Theorem~\ref{lemma:pca}}\label{alg:tight}
{\begin{minipage}{\textwidth}
\begin{tabbing}
\textbf{Input:} \quad\=A weighted set $(P,w)$ of $n$ points in $\mathbb{R}^d$, a point $p\in P$,  \\\> an integer $k\in\br{0,\cdots,d-1}$, and an error parameter $\Tau\in(0,1)$. \\
\textbf{Output:} An additive $\Tau$-approximation $\tilde{s}$ to the sensitivity $s(p)$ of $p$.
\end{tabbing}\end{minipage}}

\nl$\psi := \left(\frac{\Tau}{2}\right)^2$

\nl$\displaystyle{r := 1+ \max_{p\in P}\frac{D(p,0_d)}{\psi\Tau^2}}$\label{lineneeded}

\nl$p' := (p\mid r)$  \quad\quad\tcp{$p'\in R^{d+1}$ is a concatenation of $p\in\REAL^d$ and $r\in\REAL$.}

\For {every $q \in P$}{
	
	 $q':=(q\mid r)$

	 $w(q') := w(q)$

}

$P' := \br{ q' \mid q \in P }$\label{definp}

$s' := \nonaafinebound(P',w,p', k+1,\Tau)$ \quad\quad\tcp{See Algorithm~\ref{alg:tight1}.}\label{defs}

$\tilde{s}:= s'+80\Tau $

\Return $\tilde{s}$
\end{algorithm}

\begin{theorem} \label{lemma:pca} Let $(P,w)$ be a weighted set of $n$ points in $\REAL^d$, $p\in P$, $\Tau\in(0,\frac{1}{12}]$ be an error parameter, and let $k\in \br{0,\cdots,d-1}$ be an integer. Let $\mathcal{S}^{A}_{d}$ denote the set of all affine $k$-subspaces in $\REAL^d$, $\displaystyle{s(p)=\sup_{S\in \mathcal{S}^{A}_{d}}\frac{w(p) D(p,S)}{\sqi w(q) D(q,S)}}$ denote the sensitivity of $p$ in the query space $(P,w,\mathcal{S}^{A}_{d},D)$, and let $\tilde{s}$ be the output of a call to $\afinebound(P,w,p, k,\Tau)$; See Algorithm~\ref{alg:tight}. Then
\[
s(p)\leq \tilde{s}\leq s(p)+161\Tau.
\]
\end{theorem}

\begin{proof}
Let $S\in \mathcal{S}^{A}_{d}$. Notice that in this theorem $z=2$ and $D(p,S)=D_2(p,S)=\dist^2(p,S)$. Let  $\mathcal{S}_{d+1}$ denote the set of all non-affine $(k+1)$-subspaces of $\REAL^{d+1}$ and let $\displaystyle{s(p') = \sup_{S'\in \mathcal{S}_{d+1}}\frac{w(p)D(p',S')}{\sqi w(q) D(q',S')}}$. First by the definition of the set $P'$ at Line~\ref{definp} of Algorithm~\ref{alg:tight} and using Lemma~\ref{affinetononaffinecor} we have that for every point $p\in P$ and its corresponding $p'=(p\mid r)\in P'$ the following hold
\begin{align*}
 |s(p') - s(p)| \leq 16(2^z +1) \Tau \cdot s(p') \leq   80\Tau \cdot s(p') \leq 80\Tau,
\end{align*}
where the last inequality holds since the sensitivity is always bounded by $1$ (i.e., $s(p')\leq1$). From the previous inequality we get
\begin{align}
s(p)-80 \Tau \leq s(p') \leq s(p) + 80\Tau. \label{pahes1}
\end{align}
 Let $s'$ be the output of a call to $\nonaafinebound(P',w,p', k+1,\Tau)$ as defined at Line~\ref{defs} of Algorithm~\ref{alg:tight}. By Lemma~\ref{lemma:svd} we have
\begin{align}
s(p') \leq s' \leq s(p') + \Tau. \label{pahes2}
\end{align}
Combining~\eqref{pahes1} and~\eqref{pahes2} yields
$$s(p) \leq s(p')+80\Tau \leq  s' + 80\Tau \leq s(p') + 81\Tau \leq s(p) + 161\Tau, $$
where the first inequality holds by adding $80\Tau$ to both sides of the left hand side inequality in~\eqref{pahes1}, the second and the third inequalities are by~\eqref{pahes2}, and the third holds by adding $81\Tau$ to both side of the right hand side inequality in~\eqref{pahes1}. Considering the returned value $\tilde{s} = s' + 80\Tau$  proves the theorem as,
$$s(p) \leq \tilde{s}\leq s(p) + 161\Tau.$$
\end{proof}

\section{Experimental Results}\label{emp}
In this section we run benchmarks on real-world databases and compare our sampling algorithm with existing ones.

\paragraph{Algorithms. } We implemented the following sampling algorithms (distributions) for computing a coreset of $n$ points where every point is sampled with probability that is proportional to: (i) $1/n$ (uniform), (ii) existing sensitivity upper bound which is the sensitivity sampling algorithm of~\cite{DBLP:journals/corr/abs-1209-4893} that is mentioned at ``Introduction" section, and (iii) our tight sensitivity bound (Algorithm~\ref{alg:tight1} for SVD, and Algorithm~\ref{alg:tight} for PCA). 

\paragraph{Software and Hardware. }
 We implemented those algorithms in Python 3.6 using the libraries Numpy~\cite{oliphant2006guide} and Scipy~\cite{SciPy}. We then run experimental results that we summarize in this section.  The experments where done using Intel i7-6850K CPU @ 3.60GHZ and 64GB RAM.

\subsection{Experimental Results for $k$-PCA}\label{pcaexp}
\paragraph{Datasets. } We used the following two datasets from~\cite{anguita2013public}:
(i) Gyroscope data, which we call  ``Gyro" in our graphs.
(ii) Embedded accelerometer data ($3$-axial linear accelerations) which we call ``Acc" in our graphs.
The data sets are resulted from experiments that have been carried out with a group of $30$ volunteers within an age bracket of $19$-$48$ years. Each person performed six activities (WALKING, WALKING UPSTAIRS, WALKING DOWNSTAIRS, SITTING, STANDING, LAYING) while wearing a smartphone (Samsung Galaxy S II) on the waist. Using its embedded gyroscope, $3$-axial angular velocities were recorded, at a constant frequency of $50$Hz. The experiments have been video-recorded to label the data manually.
Data was collected from $n=7352$ measurements. Each instance consists of measurements from $3$ dimensions, $x$, $y$, $z$, each in a size of $d=128$.  The results are those the corresponding $3$ datasets.

\paragraph{The experiment. } We ran Algorithms (i)-(iii) on the above datasets in order to compute sensitivities and sample coresets of variant sizes between $1000$ to $7000$. For each coreset, we computed the sum $OPT_k(A)$ of squared distances from the rows of the input matrix $A$ to the affine $k$-subspace that minimizes this sum. We then computed this sum to the optimal solution $OPT_k(C)$ on the coreset (to the rows of $A$). The approximation error $\eps\in (0,1)$ was then defined to be $1-OPT_k(C)/OPT_k(A)$.


We used two values of $k= 5$ and $k=10$, and run each experiment 50 times.
 Results for the gyroscope data  (dataset (i)) are presented in Fig.~\ref{F1} and results for the accelerometer data (dataset (ii)) are presented in Fig.~\ref{F2}.

 	\begin{figure}[h!]
		\begin{subfigure}[h]{\sca\textwidth}
		\centering
		\includegraphics[scale=\sca]{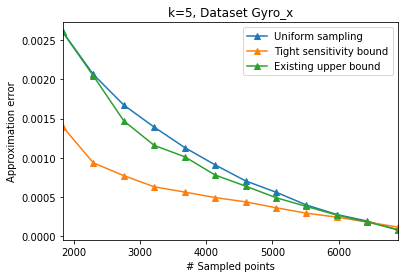}
	\end{subfigure}
		\begin{subfigure}[h]{\sca\textwidth}
		\centering
		\includegraphics[scale=\sca]{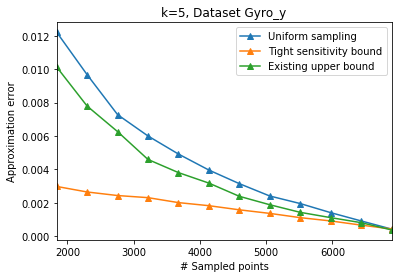}
	\end{subfigure}
	\begin{subfigure}[h]{\sca\textwidth}
		\centering
		\includegraphics[scale=\sca]{small_data/gyroxk5.png}
	\end{subfigure}

	\begin{subfigure}[h]{\sca\textwidth}
		\centering
		\includegraphics[scale=\sca]{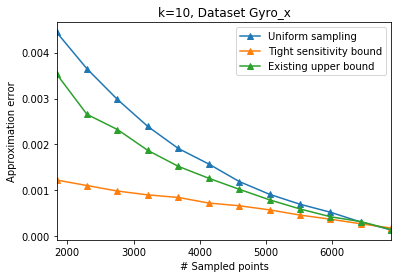}
	\end{subfigure}
		\begin{subfigure}[h]{\sca\textwidth}
		\centering
		\includegraphics[scale=\sca]{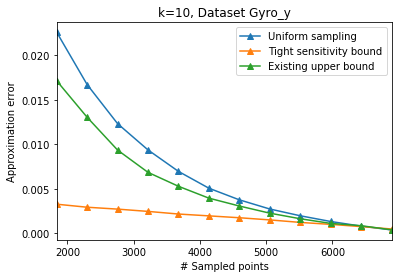}
	\end{subfigure}
	\begin{subfigure}[h]{\sca\textwidth}
		\centering
		\includegraphics[scale=\sca]{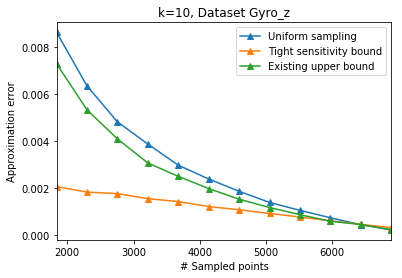}
	\end{subfigure}

\caption{\small\label{F1} Experimental results of Subsection~\ref{pcaexp} for the gyroscope data.}
\end{figure}

	\begin{figure}[h!]
		\begin{subfigure}[h]{\sca\textwidth}
		\centering
		\includegraphics[scale=\sca]{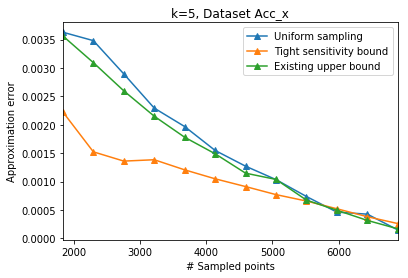}
	\end{subfigure}
		\begin{subfigure}[h]{\sca\textwidth}
		\centering
		\includegraphics[scale=\sca]{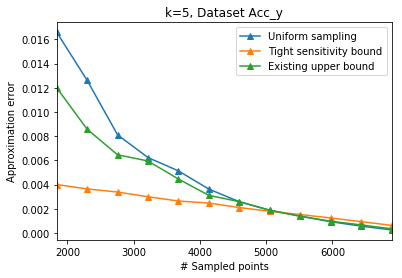}
	\end{subfigure}
	\begin{subfigure}[h]{\sca\textwidth}
		\centering
		\includegraphics[scale=\sca]{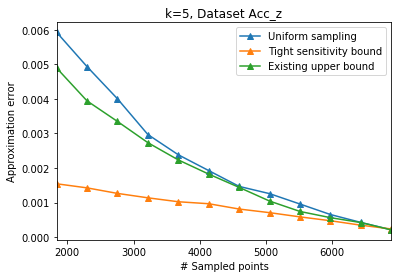}
	\end{subfigure}

	\begin{subfigure}[h]{\sca\textwidth}
		\centering
		\includegraphics[scale=\sca]{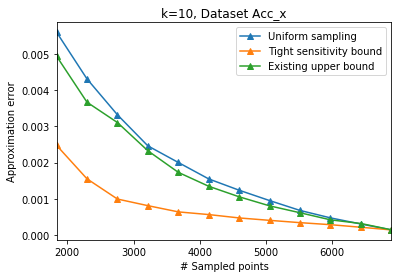}
	\end{subfigure}
		\begin{subfigure}[h]{\sca\textwidth}
		\centering
		\includegraphics[scale=\sca]{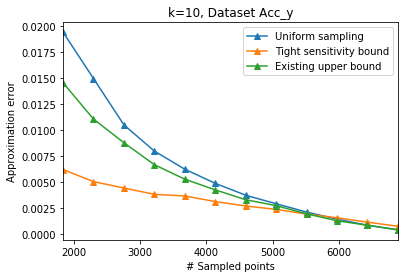}
	\end{subfigure}
	\begin{subfigure}[h]{\sca\textwidth}
		\centering
		\includegraphics[scale=\sca]{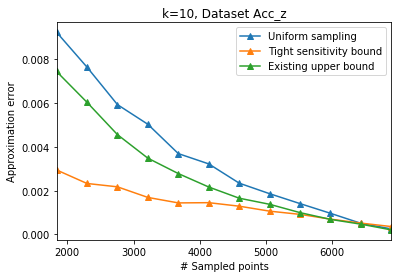}
	\end{subfigure}

\caption{\small\label{F2}Experimental results of Subsection~\ref{pcaexp} for the accelerometer data.}
\end{figure}

\subsection{Experimental Results for $k$-SVD}\label{svdexp}

\paragraph{Dataset}
We downloaded the document-term matrix of the English Wikipedia from~\cite{wic2019}, a sparse matrix of $4,624,611$ rows that correspond to documents, and $100$k columns (the dictionary of the $100$k most common words in Wikipedia~\cite{dic2012}). The entry in the $i$th row and $j$th column of this matrix is the number of how many appearances word number $j$ has in article number $i$. We call this data set ``Wiki" in our graphs. 

\paragraph{Handling large data.} To handle this large dataset in memory, we maintain the coreset for the streaming set of rows, one by one, via the common merge and reduce tree that is usually used for this purpose; see e.g.\cite{feldman2010coresets} for details.

\paragraph{The experiment. } 
We ran Algorithms (i)-(iii) on the document-term matrix of the English Wikipedia dataset in order to compute sensitivities and sample coresets of variant sizes between $1000$ to $7000$. For each coreset, we computed the sum $OPT_k(A)$ of squared distances from the rows of the input matrix $A$ to the (non-affine) $k$-subspace that minimizes this sum. We then computed this sum to the optimal solution $OPT_k(C)$ on the coreset (to the rows of $A$). The approximation error $\eps\in (0,1)$ was then defined to be $1-OPT_k(C)/OPT_k(A)$.
 We ran with different values of $k$: $1$,$3$,$9$ and $11$.

\begin{figure}[h!]
	\begin{subfigure}[h]{0.5\textwidth}
		\centering
		\includegraphics[scale=0.5]{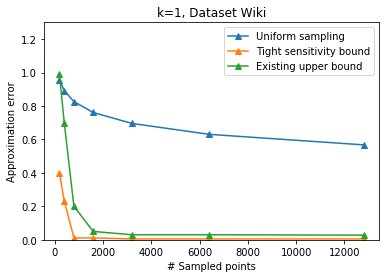}
	\end{subfigure}
	\begin{subfigure}[h]{0.5\textwidth}
		\centering
		\includegraphics[scale=0.5]{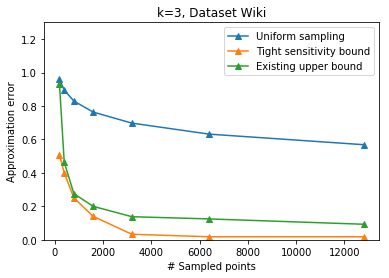}
	\end{subfigure}
	\begin{subfigure}[h]{0.5\textwidth}
		\centering
		\includegraphics[scale=0.5]{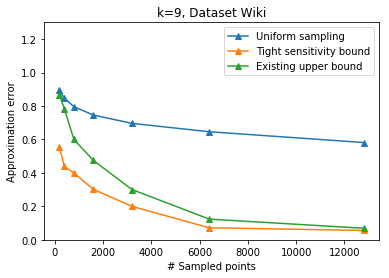}
	\end{subfigure}
	\begin{subfigure}[h]{0.5\textwidth}
		\centering
		\includegraphics[scale=0.5]{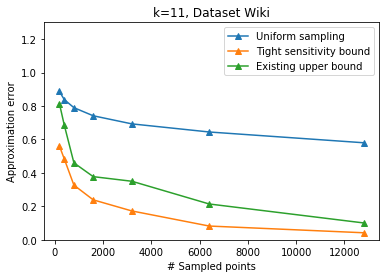}
	\end{subfigure}
\caption{\small\label{F5}Experimental results of Subsection~\ref{svdexp} for the document-term matrix of the English Wikipedia dataset.}
\end{figure}

\subsection{Conclusions and open problems}
We presented algorithms to compute exact sensitivity bounds for the $k$-SVD query spaces. Since the size of the coresets depends on the total sensitivity, we obtained coresets of size smaller and data dependent compared to existing worst-case upper bounds. We then suggested a generic reduction that enables us to generate tight sensivities also for the $k$-PCA problem (for affine $k$-subspaces).

Our experimental results show that our coreset indeed always smaller in practice.
We hope that the presented approach and open code would help to compute tight sensitivities for many other problems such as $k$-clustering, and other machine/deep learning problems.

\clearpage

\clearpage
\bibliographystyle{alpha}
\bibliography{tight_bounds}
\end{document}